\documentclass[nohyperref]{article}

\usepackage{microtype}
\usepackage{graphicx}
\usepackage{subfigure}
\usepackage{booktabs} % for professional tables
\PassOptionsToPackage{hyphens}{url}
\usepackage{hyperref}

\usepackage[accepted]{icml2022}

\usepackage{amsmath}
\usepackage{amssymb}
\usepackage{mathtools}
\usepackage{amsthm}
\usepackage{thmtools, thm-restate}
\usepackage{stmaryrd}

\usepackage[capitalize,noabbrev]{cleveref}

\theoremstyle{plain}
\newtheorem{theorem}{Theorem}[section]
\newtheorem{proposition}[theorem]{Proposition}

\theoremstyle{definition}
\newtheorem{definition}[theorem]{Definition}
\newtheorem{assumption}[theorem]{Assumption}
\theoremstyle{remark}
\newtheorem{remark}[theorem]{Remark}
\newtheorem{example}[theorem]{Example}

\usepackage[textsize=tiny]{todonotes}
\usepackage{dsfont}
\usepackage[T1]{fontenc}

\usepackage{listings}

\definecolor{codegreen}{rgb}{0,0.6,0}
\definecolor{codegray}{rgb}{0.5,0.5,0.5}
\definecolor{codeorange}{rgb}{230, 96, 0}
\definecolor{backcolour}{rgb}{0.95,0.95,0.92}

\lstdefinestyle{mystyle}{
    backgroundcolor=\color{backcolour},
    commentstyle=\color{codegreen},
    keywordstyle=\color{blue},
    numberstyle=\tiny\color{codegray},
    basicstyle=\ttfamily\footnotesize,
    breakatwhitespace=false,  
    breaklines=true,                 
    captionpos=b,                    
    keepspaces=true,                 
    numbers=left,                    
    numbersep=5pt,                  
    showspaces=false,                
    showstringspaces=false,
    showtabs=false,                  
    tabsize=2,
    emph = {forward, loss, train, predict, __init__},
    emphstyle=\color{purple}\ttfamily
}
\icmltitlerunning{Graph prediction with Gromov-Wasserstein barycenter}

\newcommand{\R}{\mathbb{R}}
\newcommand{\N}{\mathbb{N}}

\newcommand{\E}{\mathbb{E}}

\newcommand{\bmH}{\mathcal{H}}
\newcommand{\bmK}{\mathcal{K}}
\newcommand{\bmG}{\mathcal{G}}
\newcommand{\bmX}{\mathcal{X}}
\newcommand{\bmL}{\mathcal{L}}

\newcommand{\bmY}{\mathcal{Y}}
\newcommand{\bmZ}{\mathcal{Z}}
\newcommand{\bmF}{\mathcal{F}}
\newcommand{\bmM}{\mathcal{M}}

\newcommand{\bmU}{\mathcal{U}}
\newcommand{\risk}{\mathcal{R}}
\newcommand{\emprisk}{\hat{\risk}}
\newcommand{\plans}{\mathcal{P}}

\newcommand{\un}{\mathds{1}}
\newcommand{\FGW}{\text{FGW}}
\newcommand{\GW}{\text{GW}}

\DeclareMathOperator{\hs}{HS}

\DeclareMathOperator*{\argmin}{arg\,min\,}
\newcommand\subsetsim{\mathrel{%
  \ooalign{\raise0.2ex\hbox{$\subset$}\cr\hidewidth\raise-0.8ex\hbox{\scalebox{0.9}{$\sim$}}\hidewidth\cr}}}
  
\newcommand{\rf}[1]{ {\color{black}#1}}

\newcommand{\cam}[1]{\textcolor{black}{#1}}

\begin{document}

\twocolumn[
\icmltitle{Learning to Predict Graphs with Fused Gromov-Wasserstein Barycenters}



\icmlsetsymbol{equal}{*}

\begin{icmlauthorlist}
\icmlauthor{Luc Brogat-Motte}{tel}
\icmlauthor{Rémi Flamary}{poly}
\icmlauthor{Céline Brouard}{inrae}
\icmlauthor{Juho Rousu}{aalto}
\icmlauthor{Florence d'Alché-Buc}{tel}
\end{icmlauthorlist}

\icmlaffiliation{tel}{LTCI, Télécom Paris, Institut Polytechnique de Paris, France}
\icmlaffiliation{poly}{Ecole Polytechnique, Institut Polytechnique de Paris, CMAP, UMR 7641, Palaiseau, France}
\icmlaffiliation{inrae}{Université de Toulouse, INRAE, UR MIAT, France}
\icmlaffiliation{aalto}{Department of Computer Science, Aalto University, Finland}
\icmlcorrespondingauthor{Luc Brogat-Motte}{luc.motte@telecom-paris.fr}

\icmlkeywords{Machine Learning, ICML}

\vskip 0.3in
]



\printAffiliationsAndNotice{}  % leave blank if no need to mention equal contribution

\begin{abstract}
This paper introduces a novel and generic framework to solve the flagship task
of supervised labeled graph prediction by leveraging  Optimal Transport tools. 
%Undirected graphs are
%represented by adjacency matrices and probability distribution on the set of
%nodes while labels are features associated to nodes. 
We formulate the problem as regression with the Fused Gromov-Wasserstein (FGW) loss and propose a predictive model relying on a FGW barycenter whose weights depend on inputs. First we introduce a non-parametric estimator based on kernel ridge regression for which theoretical results such as consistency and excess risk bound are proved. Next we propose an interpretable parametric model where the barycenter weights are modeled with a neural network and the graphs on which the FGW barycenter is calculated are additionally learned. Numerical experiments show the strength of the method and its ability to interpolate in the labeled graph space on simulated data and on a difficult metabolic identification problem where it can reach very good performance with very little engineering.

% While enjoying theoretical guarantees, the method
% is studied empirically, with very good performance. \rf{RF: TODO add details on
% experiments here when we have them}.

\end{abstract}

\section{Introduction}\label{sec:intro}

Graphs allow to represent entities and their
interactions. They are ubiquitous in real-world: social networks, molecular structures, biological protein-protein networks, recommender systems, are naturally represented as graphs. Nevertheless, graphs structured data can be challenging to process. An important effort has been made to design well-tailored machine learning methods for graphs. For example, many kernels for graphs have been proposed allowing to perform graph classification, graph clustering, graph regression \citep{kriege2020survey}. Many deep learning architecture have also been developped \citep{zhang2020deep}, including Graph Convolutional Networks (GCNs) that are powerful models for processing graphs. 
%They have, for example, been used for learning representations of molecular graphs tailored to the task at hand directly \citep{duvenaud2015convolutional}, avoiding the use of pre-computed graphs' representation, such as molecular fingerprints %\citep{todeschini2008handbook, rogers2010extended}(binary vector indicating the chemical substructures of a molecule). %Nevertheless, GCNs training often require a large quantity of labeled data, which can be very expensive, and is sometimes not available as for example with the metabolite identification problem \citep{brouard2016fast}.

Most of existing works in machine learning consider graphs as inputs, but predicting a graph as output given an input from an arbitrary input space has received much less attention.
In this work, we target the difficult problem of supervised learning of graph-valued functions. In contrasts with node classification \citep{bhagat2011node}, or link prediction \citep{lu2011link}, entire graphs are predicted.  %propose an original alternative to these sequential models and 
Supervised Graph Prediction (SGP)  can be considered as an emblematic instance of Structured Prediction (SP) with the difficulty that the output space is of finite but huge cardinality and contains structures of different sizes. In principle, any of the three main approaches to SP, energy-based models, surrogate approaches and end-to-end learning, are eligible. In energy-based models \citep{tsochantaridis2005large,chenb15,belanger2016structured}, predictions are obtained by maximizing a score function for input-output pairs over the output space. In surrogate approaches \citep{cortes05,geurts2006, brouard2016input,ciliberto2016consistent}, a feature map is used to embed the structured outputs. After minimizing a surrogate loss a decoding procedure is used to map back the surrogate solution. End-to-end learning methods attempt to solve structured prediction by directly learning to generate a structured object \citep{belanger2017,silver2017predictron} and leverage differentiable and relaxed definition of energy-based methods (see for instance \citet{pillutla2018smoother,mensch2018diff}). 

Nevertheless, to our knowledge, among surrogate methods, only Input Output Kernel Regression (IOKR) \citep{brouard2016input} that leverages kernel trick in the output space has been successfully applied to SGP while on the side of end-to-end learning, several generative models allow to build and predict graphs but in general in an unsupervised setting. \citet{gomez2018automatic} try to obtain a continuous representation of molecules using a variational autoencoding (VAE) of text representations of molecules (SMILES). \citet{kusner2017grammar} incorporates in the VAE architecture knowledge about the structure of SMILES thanks to its available grammar. \citet{olivecrona2017molecular, liu2017retrosynthetic, li2018learning, you2018graph, shi2020graphaf} propose models that generate graphs using a sequential process generating one node/edge at a time, and train it by maximizing the likelihood.

In supervised graph prediction, the crucial issue is to learn or leverage appropriate representations of graphs, a problem tightly linked with the choice of a loss function.
%In the case of supervised graph prediction, major challenges come from the fact that the number of possible outputs can be extremely large and that the graphs have generally different sizes.
%Finding a good loss and output representation is therefore particularly crucial. 
Typical graph representations usually rely on graph kernels leveraging fingerprint representations, i.e. a bag of motifs approach \citep{ralaivola}, or more involved kernels such the Weisfeiler-Lehman kernel \citep{shervashidze2011weisfeiler}. 
%used a structured prediction approach for a specific supervised graph prediction problem, consisting in predicting molecular graph from tandem mass spectra. The problem was simplified by learning a similarity between molecular fingerprints and searching the most similar fingerprint to the prediction among a set of molecular candidates.
In this work, we propose to exploit another kind of graph representation, opening the door to the use of an Optimal Transport loss, and derive an end-to-end learning approach that constrasts to energy-based learning and surrogate methods.

\rf{Successful applications of optimal transport (OT) in machine learning are becoming increasingly numerous thanks to the advent of numerical optimal transport \citep{cuturi2013sinkhorn, altschuler2017near, peyre2019computational}. Examples include domain adaptation \citep{courty2016optimal}, unsupervised learning \citep{arjovsky2017wasserstein}, multi-label classification \citep{frogner2015learning}, natural language processing \citep{kusner2015word}, fair classification \citep{gordaliza2019obtaining}, supervised representation learning \citep{flamary2018wasserstein}. Optimal transport provide meaningful distances between probability distributions, by leveraging the geometry of the underlying metric spaces.\\ %The numerical complexity of using OT has been greatly improved by the recent advances in computational optimal transport \citep{cuturi2013sinkhorn, peyre2019computational, altschuler2017near} which made possible its use on the aforementioned applications.\\
Supervised learning with optimal transport losses has been considered in \citet{frogner2015learning, bonneel2016wasserstein, luise:hal-01958887,mensch2019geometric} for predicting histograms. But traditional OT loss can be applied only between distributions lying in the same space, preventing their use on structured data such as graphs.  \citet{menoli2011} proposed  the 
Gromov-Wasserstein distance that can measure similarity between metric measure space and has been used as a distance between graphs in several applications such as computing graph barycenters \cite{peyre2016gromov} or for performing graph node embedding \cite{xu2019gromov} and graph partitioning \cite{xu2019scalable}. This distance has been extended to the Fused Gromov-Wasserstein distance (FGW) in \citet{vayer2019optimal,vayer2020fused} with applications to attributed graphs  classification, barycenter estimation and more recently dictionary learning \cite{vincent2021online}. Those novel divergences that can be used on graphs are a natural fit, first as a loss term in graph prediction but also as a way to model the space of graphs for instance using FGW barycenters.}

\paragraph{Contributions.}
\rf{In this paper we present the following novel contributions. First we propose a novel and and general framework in Sec. \ref{sec:problem-setting} for graph prediction building on FGW as a loss and FGW barycenter as a way to interpolate in the target space. The framework is studied theoretically in Sec. \ref{sec:nonparam} in the non-parametric case for which we provide consistency and excess risk bounds. Then a parametric version of the model building on deep neural network and learning of the template graphs is proposed in Sec. \ref{sec:nns} with a simple stochastic gradient algorithm. Finally we provide some numerical experiments in Sec. \ref{sec:num_exp} on synthetic and real life metabolite prediction datasets.}
% \vspace{-1em}
% \begin{itemize}
%   \setlength{\itemsep}{1pt}
%   \setlength{\parskip}{0pt}
%   \setlength{\parsep}{0pt}
%     \item A novel framework for structured prediction relying on the
%     Gromov-Wassrestein (GW) and Fused Gromov Wasserstein (FGW) as a data fitting
%     term Sec. \ref{sec:problem-setting}.
%     \item A modeling for the score function in structured graph
%     prediction relying on conditional GW and FGW barycenters, providing a
%     continuous prediction space.
%     \item A theoretical study of the framework in the non-parametric case with
%     consistency and excess risk bounds.
%     \item A parametric version of the prediction models using neural network
%     embeddings learned with gradient descent.
%     \item Numerical experiments illustrating both non-parametric and parametric
%     models on simulated data and real life metabolite identification problem.
% \end{itemize}

% on change
\section{Background on OT for graphs}\label{sec:background}
%\section{Problem setting}

We begin by introducing how to represent graphs and define distances between graph by leveraging the Fused Gromov-Wasserstein distance.

\paragraph{Notations.} $\un_n$ is the all-ones vector with size $n$. $\delta_x$ denotes the Dirac measure in $x$ for $x$ in a measurable space. Identity matrix in $\R^{N \times N}$ is noted $I_N$. $\mathcal{L}(\mathcal{A})$ the set of bounded linear operator from $\mathcal{A}$ to $\mathcal{A}$.  $\mathcal{M}(\mathcal{A}, \mathcal{B})$ the set of measurable functions from $\mathcal{A}$ to $\mathcal{B}$.

\paragraph{Graph represented as metric measure spaces.} Denote $n_{max} \in \mathbb{N}^*$ the maximal number of nodes (vertices) in the graphs we consider in this paper. We define $\bmF \subset \R^d$ a finite feature space of size $| \bmF | < \infty$. A labeled graph $y$ of $n \leq n_{max}$ nodes is represented by a triplet $y=(C,F,h)$ where $C=C^T \in \{0,1\}^{n \times n}$ is the adjacency matrix, and $F=(F_i)_{i=1}^n$ is a $n$-tuple composed of feature vectors $F_i \in \bmF \subset \R^d $ labeling each node indexed by $i$. The space of labeled graphs $\bmY$ is thus defined as $\bmY = \{(C, F, h) \,|\, n \leq n_{max}, C \in \{0,1\}^{n \times n}, C^T = C, F=(F_i)_{i=1}^n \in \bmF^{n}, h= \frac{1}{n}\un_{n}\}$.
Observe that we equipped all graphs with a uniform discrete probability distributions over the nodes \rf{$\mu= \sum_{i=1}^n h_i \delta_{u_i}$ where $u_i= (v_i, F_i)$  represents the structure $v_i$ (encoded only through $C(i,j),\,\forall j$) and the feature information $F_j$ attached to a vertex $i$ \cite{vayer2019optimal}.}
%node representation and their relations are encoded in $C(i,j)$ \cite{peyre2016gromov}. 
These weights indicate the relative importance of the vertices in the graph. In absence of this information, we simply fix uniform weights $h_i=\frac{1}{n}$ for a graph of size $n$. Now, let us introduce the space of continuous relaxed graphs with {\bf fixed size} $n$: $\bmZ_n = \{(C, F, h) \,|\, C \in [0,1]^{n \times n}, C^T = C, F \in \text{Conv}(\bmF)^n, h= n^{-1}\un_{n}\}$. $\text{Conv}(\bmF)$ denotes the convex hull of $\bmF$ in $\R^{n \times d}$. We call $\bmZ= \bigcup (\bmZ_i)_{i=1}^{n_{max}}$ \rf{and want to emphasize that $\bmY\subset\bmZ$.}

\paragraph{Gromov-Wasserstein (GW) distance. }\label{subsec:method} The Gromov-Wassertein distance between  metric measure space has been introduced by \citet{menoli2011} for object matching. The GW distance defines an OT problem to compare these objects, with the key property that it defines a strict metric on the collection of isomorphism classes of metric measure spaces.  In this paper, we adopt this angle to address graph representation and graph comparison, opening the door to define a loss for supervised graph prediction.
Let  $z_1 = (C_1, n_{1}^{-1}\mathds{1}_{n_{1}})$ and $z_2 = (C_2, n_{2}^{-1}\mathds{1}_{n_{2}})$ be the representation of two  graphs with respectively $n_1 \in
\N^*$ and $n_2 \in \N^*$ nodes, the Gromov-Wasserstein (GW) distance between
$z_1$ and $z_2$, $\GW^2_2(z_1, z_2)$, is defined as follows:
\begin{align}\small
     \min\limits_{\pi \in \plans_{n_{1},n_{2}}} \sum_{i,k=1}^{n_1} \sum_{j,l=1}^{n_2} (C_1(i,k) - C_2(j,l))^2 \pi_{i,j}\pi_{k,l},
\end{align}
where $\plans_{n_1,n_2} = \{ \pi \in \R^{n_1 \times n_2}_+ | \pi\mathds{1}_{n_{2}} = n_{1}^{-1}\mathds{1}_{n_{1}}, \pi^T\mathds{1}_{n_1}=n_{2}^{-1}\mathds{1}_{n_{2}}\}$.
$\GW_2$ can be used to compare unlabeled graphs with potentially different numbers of nodes, it is symmetric, positive and satisfies the triangle inequality. Furthermore, it
is equal to zero when $z_1$ and $z_2$ are isomorphic, namely when there exist a
bijection $\phi:  \llbracket 1, n_1 \rrbracket \rightarrow \llbracket 1, n_2
\rrbracket$ such that $C_2(\phi(i), \phi(j)) = C_1(i, j)$ for all $i, j \in
\llbracket 1, n_1 \rrbracket$. GW provides a distance on the unlabeled graph
quotiented by the isomorphism, making it a natural metric when comparing
graphs.%, as  in real-world problems, one wants to consider isomorphic graphs as
%equal.

\paragraph{Fused Gromov-Wasserstein (FGW) distance. } The FGW distance has been proposed recently as an extension of GW that can be used to measure the similarity between attributed graphs \citep{vayer2020fused}.
%Gromov-Wassertein distance has been recently extended into fused
%Gromov-Wasserstein distance to cope with feature information in addition to structure %information \citep{vayer2020fused}. 
 For a given $0 \leq \beta \leq 1$, the FGW distance between two labeled weighted graphs represented as $z_1= (C_1, F_1, n_1^{-1}\un_{n_1})$ and $z_2= (C_2, F_2, n_2^{-1}\un_{n_2})$ is defined as follows \citep{vayer2020fused}: 
\begin{align*}\label{eq:fgw}
    \begin{split}
        \FGW^2_2(z_1, z_2) =& \min\limits_{\pi \in \plans_{n1,n2}} \sum_{i,k,j,l} \big[(1-\beta)\|F_1(i) - F_2(j)\|_{\R^d}^2\\ 
   &+ \beta (C_1(i,k) - C_2(j,l))^2\big] \pi_{i,j}\pi_{k,l}.
    \end{split}
\end{align*}
The optimal transport plan matches the vertices of the two graphs by
minimizing the discrepancy between the labels, while preserving the pairwise
similarities between the nodes. Parameter $\beta$ governs the trade-off between structure and label information. Its choice is typically driven by the application.  %Note that FGW distance adapts to the graph space at hand through the choice of the labels' representations $\bmF$.

% Therefore, one should design $F_\theta$ such that $f_\theta(x)$ computation is tractable, for example using sequential strategies \citep{li2018learning}. Moreover, at the end of the day, supervised learning's
% success relies on methods ability to interpolate between training data. However, how to interpolate in structured space is not obvious. Taking the example of labeled graph space $\bmZ$, relevantly averaging two graphs $(C_1, F_1) \in \bmZ$ and $(C_2, F_2) \in \bmZ$ is unclear: even if the two graphs have the same number of nodes, the linear average $\left((C_1 + C_2) / 2, (F_1 + F_2) / 2\right)$ is not judicious as the adjacency matrices are in general not aligned.
% 

\section{Graph prediction with Fused Gromov-Wasserstein}\label{sec:problem-setting}
\begin{figure*}[t]
%\vskip 0.2in
\begin{center}
\centerline{\includegraphics[width=.9\textwidth]{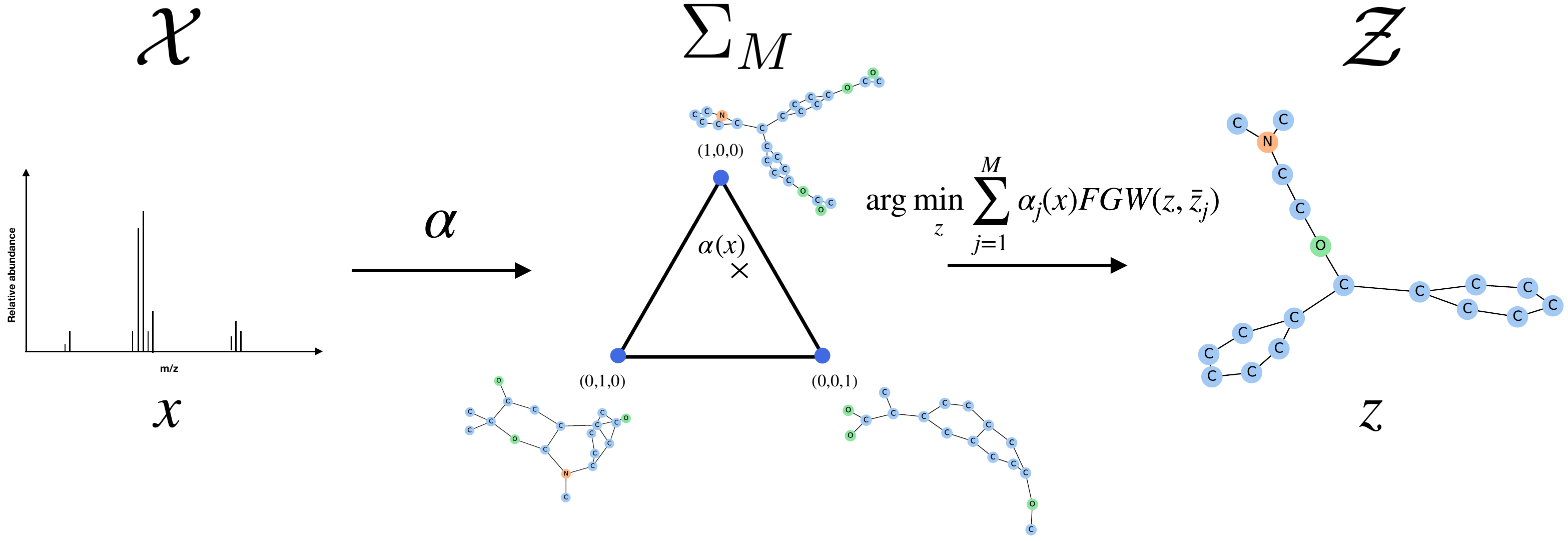}}\end{center}\vspace{-10mm}
\caption{Proposed supervised graph prediction model. The input $x$ (left) is mapped with $\boldsymbol{\alpha}(x)$ onto the simplex (center) where the weights are used for computing the prediction as a FGW barycenter (right). }\vskip-3mm
\label{fig:model}
\end{figure*}

\paragraph{Relaxed Supervised Graph Prediction.} 

In this work, we consider labeled graph prediction as a {\it relaxed} structured output prediction problem. We assume that $\bmX$ is the input space and that the predictions belong to the space $\bmZ_n$ defined in Section \ref{sec:background}, for a given value of $n$, while we observe training data in the finite set $\bmY$. We define an asymmetric partially relaxed structured loss function $\Delta: \bmZ_n \times \bmY \rightarrow \R^+$. Given a finite sample $(x_i, y_i)_{i=1}^N$ independently drawn from an unknown distribution $\rho$ on $\bmX \times \bmY$, we consider the problem of estimating a target function $f^*: \bmX \rightarrow \bmZ_n$ with values in the structured objects $\bmZ_n$ that minimizes the expected risk:
\begin{equation}\label{eq:exp_structured_risk}
    \risk_{\Delta}^n(f) = \E_{\rho}[\Delta(f(X), Y)],
\end{equation}
by an estimate $\hat{f}$ obtained by minimizing the empirical counterpart of the true risk, namely the empirical risk:
\begin{equation}\label{eq:emp-risk}
\emprisk_{\Delta}^n(f) = \sum_{i=1}^N \Delta(f(x_i), y_i),
\end{equation}
over the hypothesis space $\bmG^n \subset \bmM(\bmX, \bmZ_n)$. 
The goal of this paper is to provide a whole framework to address this family of problems instantiated by $n \leq n_{max}$. Note that the complexity of the task depends primarily on $n$.  
%If one considers different values for $n$, say $n_1 < n_2 < \ldots n_p \leq n_{max}$, and call $f_1^* \in \bmF(\bmX, \bmZ_1), f_2^*\bmF(\bmX, \bmZ_2), \ldots, f_p^*\bmF(\bmX,\bmZ_p)$ the minimizers of respectively, the true risks $\emprisk_{\Delta}^1(f), \ldots,\emprisk_{\Delta}^1(f)$, then we can compare  \fl{writing this}

\paragraph{FGW as training loss.}
We propose in this paper to use the FGW distance as the loss. More precisely, we define:
\begin{equation}\label{eq:loss-delta-fgw}
\forall (z,y) \in \bmZ_n \times \bmY,~\Delta_{\FGW}(z,y):=\FGW_2^2(z,z_y),
\end{equation}
where $z_y = (C_y,F_y,n_y^{-1}\un_{n_y}) \in \bmZ_{n_y} \subset \bmZ$ is the representation of  $y=(C_y,F_y,n_y^{-1}\un_{n_y}) \in \bmY$.
As FGW is defined for graphs of different sizes, the expression in Eq. \eqref{eq:loss-delta-fgw} is well posed.
Accordingly, for all $i=1, \ldots N$, we denote $z_i \in \bmZ_{n_i}$ the relaxed version of $y_i \in \bmY$ with number of nodes $n_i$.

\paragraph{Supervised Graph Prediction with FGW.}

Having fixed a value for $n$ and following these definitions, the empirical risk minimization problem now writes as follows.
Given the training sample $\{(x_i, z_i)_{i=1}^N\}$, we want to find a minimizer over $\bmG^n \subset \bmM(\bmX, \bmZ_n)$ of the following problem:
\begin{equation}\label{eq:emp-risk-pb}
\min_{f \in \bmG^n} \sum_{i=1}^N \FGW_2^2(f(x_i), z_i). 
\end{equation}
\cam{
\begin{remark}[Role of the graph sizes for the FGW distance] For the $\FGW$ distance, it is worth noting that graphs' sizes act as resolutions, namely levels of prevision in the description of graphs. We denote by $\overset{\FGW}{\subset}$ the subset symbol for the equivalence classes induced by the $\FGW$ metric. We approximately have, for $n < n', Z_{n} \overset{\FGW}{\subsetsim}Z_{n'}$ depending if exact or approximate resampling is possible. For instance, we exactly have, for all $n \in \mathbb{N}^*, Z_{n} \overset{\FGW}{\subset}Z_{2n}$. It means that low-resolution graphs can be represented exactly as high-resolution graphs. Conversely, one can approximate a high-resolution graph with a low-resolution graph. This property is leveraged in the model hereinafter proposed. Note that if one wants to compare two graphs, with equal weights on each node, it is still possible to do padding: add nodes with no neighbours, and with a chosen constant label.
\end{remark}
}

\paragraph{Structured prediction model.}
To address this structured regression problem, we propose a generic model
$f_\theta: \bmX \to \bmZ_n$ expressed as a {\bf conditional FGW barycenter}
computed over $M$ template graphs $\bar z_j \in \bmZ$ (See Figure \ref{fig:model}):
\begin{equation}\label{eq:fgw_estimator}
    f_\theta(x) = \argmin_{z \in \bmZ_n} \sum_{j=1}^M \alpha_j(x; W) \FGW^2_2(z, \bar z_j), 
\end{equation}
where the weights $\alpha_j(x; W): \bmX \rightarrow \R^+$ are functions that can be understood as similarity scores between $x$ and $x_j$. We include in a single parameter $\theta=(M,(\bar z_j)_{j=1}^M, W)$ all model's parameters.

A key feature of the proposed model $f_{\theta}$ is that it interpolates in the
graph space $\bmZ$ by using the Fr\'echet mean with respect to the FGW distance.
Therefore, it inherits the good properties of FGW, especially including the
invariance under isomorphism (two isomorphic graphs have equal scores in Eq.
\eqref{eq:fgw_estimator}). \rf{Moreover, in terms of computations, the proposed model
leverages the recent advances in computational optimal transport such as  Conditional Gradient descent \cite{vayer2019optimal} or Mirror descent for (F)GW with entropic regularization \cite{peyre2016gromov}.}
% Note that interestingly the graphs $z_1, \ldots z_N$  might have a smaller number of nodes than $n$, in this case the model can only approximate with the closest graph of size $n$ in the FGW sense.}

% making the
% NP-hard computation of the argmin over the space of graph approximately
% tractable thanks to a Conditional gradient descent \cite{vayer2019optimal} or
% with entropic regularization \cite{peyre2016gromov}. %\fl{To be added: ref??}
%  \fl{It is also important that the barycenter is found in $\bmZ_n$, should we discuss thta: $n$ is not necessarily enough big or at least not equal to the maximum size of graphs $z_1, \ldots z_N$ ?}

\paragraph{Properties of $f_\theta$.}
Relying on recent works that studied in a large extent GW and FGW barycenters, we now discuss the shape of the recovered objects \citep[Eq. 14]{peyre2016gromov, vayer2020fused}. The evaluation of $f_\theta$ on input $x$ writes as follows: $f_{\theta}(x) = (C(x;\theta), F(x; \theta), n^{-1} \un_{n})$, where the structure and feature barycenters are:
\begin{align}\vspace{-2mm}\label{eq:C}
C(x;\theta) &= n^2\sum_{j=1}^M \alpha_j(x; W) \bar \pi_j^T\bar C_j\bar \pi_j~\in [0,1]^{n \times n},\\
%\end{align}
%and the feature barycenter expresses as 
%\begin{align}
\label{eq:A}
    F(x; \theta) &= n\sum_{j=1}^M \alpha_j(x;W)\bar  F_j \bar \pi_j^T~\in  \R^{n \times d}.\vspace{-2mm}
\end{align}
%$C(x;\theta) \in \bmS_{n}$ and $F(x; \theta) \in \R^{n \times d}$.
%\fl{$n_out$ should be depending on $x$ ??}\\
The $(\bar\pi_j)_j$ are the optimal transport plans from $(\bar C_j, \bar F_j)_j$ to the
barycenter $(C(x;\theta), F(x;\theta))$ \citep[Eq. (8)]{cuturi2014fast} , and thus depend on $\theta$. 
Note that a very
appealing property of using FGW barycenter is that the order
$n$ (that fixes the prediction space $\bmZ_n$) of the prediction does not depend on the parameters $\theta$. This means that a unique trained model can predict several objects with a different resolution $n$ %depending on input $x$, 
allowing better interpretation at small
resolution and finer modeling at higher resolution. This will be illustrated in
the experimental section.
%\textbf{}\fl{more clarity here, regarding $\bmZ_n$}% \fl{so we should put $n^x $ ??}
\\
In the next sections, we propose two different approaches to learn and define
the conditional barycenter. The first one in Section \ref{sec:nonparam} leads to a purely nonparametric
estimator with $M= N$ and $\bar z_j = z_j$ and the second one proposed in Section
\ref{sec:nns} relies on a deep
neural network for the weight functions $\alpha_j$s' while the template graphs
$(\bar z_j)_{j=1}^M$ are learned as well.

\section{Nonparametric conditional Gromov-Wasserstein barycenter}
\label{sec:nonparam}

\paragraph{Non-parametric estimator with kernels.}
Before addressing the general problem of learning both the template graphs and
the weight function $\alpha$, we adopt a nonparametric point of view to address
the structured regression problem. Under some conditions we recover a FGW conditional
barycenter estimator of the following form:
\begin{align}\label{eq:nonpar_fgw_estimator}
    f_W(x) = \argmin_{z \in \bmZ_n} \sum_{j=1}^N \alpha_j(x; W) \FGW^2_2(z, z_j), 
\end{align}
where $\theta = W $ is now the single parameter to learn and the template graphs $\bar z_j$  are not estimated but set as all the training samples $z_j$.
Similarly to scalar or vector-valued regression, one can find many different ways to define the weight functions $\alpha_i$ in the large family of nonparametric estimators \citep{geurts2006, ciliberto2020general}. We propose here a kernel approach that leverages kernel ridge regression.

Defining a positive definite kernel on the input space $k: \bmX \times \bmX \rightarrow \R$, one can consider the coefficients of kernel ridge estimation as in \citet{brouard2016input,ciliberto2020general} to define the weight function $\alpha: \bmX \to \R^N$: 
\begin{equation}\label{eq:krr_weights}
    \alpha(x) = (K + \lambda I_N )^{-1} k_x 
\end{equation}
with the Gram matrix $K = (k(x_i, x_j))_{ij} \in \R^{N \times N}$ and the vector $k_x^T=(k(x, x_1), \dots, k(x, x_N))$. Such a model leverages learning in vector-valued Reproducing Kernel Hilbert Spaces and
is rooted in the Implicit Loss Embedding (ILE) framework proposed and studied by \citet{ciliberto2020general}. %\\
\begin{example}
In the metabolite identification problem (see Section \ref{sec:num_exp}), the input takes the form of tandem mass spectra.
%, that can be represented as two-dimensional peak tuples. 
A typical relevant kernel $k$ for such data is the probability product kernel (PPK) \citep{heinonen2012}.
\end{example}

\subsection{Theoretical justification for the proposed model}\label{sec:theory}

%\fl{Continuity:}\\
%\fl{Note $\Delta_y(z) = FGW(z_y,z)$ }\\
%\fl{$\forall y \in \bmY$, $\forall \delta  > 0$, there exists $\epsilon > 0$ such that:~ if $FGW(z_1,z_2) < \delta$, then $\| \Delta_y(z_1) - \Delta_y(z_2)\| < \epsilon$.}\\
%\subsection{Justification of the proposed model}

The framework SELF \citep{ciliberto2016consistent} and its extension ILE \citep{ciliberto2020general} concerns general regression problems defined by an asymmetric loss $\Delta: \bmZ \times \bmY \rightarrow \R$ that can be written using output embeddings, allowing to solve a surrogate regression problem in the output embedding space. We recall the ILE property and the resulting benefits, especially when working in vector-valued Reproducing Kernel Hilbert Space. %We then show that FGW satisfies the ILE property under mild assumptions and retrieve the model in Eq. \eqref{eq:nonpar_fgw_estimator}.

%\paragraph{ILE framework}
\begin{definition}[ILE] For given spaces $\bmZ, \bmY$, a map $\Delta: \bmZ \times \bmY \rightarrow \R$ is said to admit an Implicit Loss Embedding (ILE) if there exists a separable Hilbert space $\bmU$ and two measurable bounded maps $\psi: \bmZ \rightarrow \bmU$ and $\phi: \bmY \rightarrow \bmU$, such that for any $z \in \bmZ, y \in \bmY$:~
$\Delta(z, y) = \langle \psi(z),\, \phi(y)\rangle_{\bmU}.$
\end{definition}
Note that this definition highlights an asymmetry between the processing of $z$ and $y$. A regression problem based on a loss satisfying the ILE condition enjoys interesting properties. 
The following true risk minimization problem:
%\begin{align*}\label{eq:true-risk-delta}
$\min_f \E_{\rho}[\Delta(f(X), Y)]:= \E_{\rho}[ \langle \psi(f(X)),\, \phi(Y)\rangle_{\bmU}],$
%\end{align*}
can be converted into i) a surrogate (intermediate) and simpler least-squares regression problem into the implicit embedding space $\bmU$, i.e.  $\min_{h : \bmX \rightarrow \bmU} \E_{\rho}[\| h(X) - \phi(Y) \|^2_{\bmU}]$, and ii) a decoding phase: $f^*(x) := \argmin_z \langle \psi(z),\, h^*(x)\rangle_{\bmU},$~ where $h^*$ is solution of problem i), i.e. $h^*(x) = \E[\phi(Y)|x]$.
A nice property proven by \citet{ciliberto2020general} is the one of Fisher
consistency,  $f^*$ is exactly the minimizer of problem in Eq.
\eqref{eq:exp_structured_risk}, justifying the surrogate approaches.

\paragraph{Structured prediction with implicit embedding and kernels.}
Assuming the loss $\Delta$ is ILE, when relying on a i.i.d. training sample
$\{(x_i,y_i)_{i=1}^N\}$, one gets $\hat{h}$ an estimator of $h^*$  by minimizing
the corresponding (regularized) empirical risk and then builds $\hat{f}$. 

If we choose to search $\hat{h}$ in the vector-valued Reproducing Kernel Hilbert Space $\bmH_{\bmK}$ associated to the decomposable operator-valued kernel $\bmK: \bmX \times \bmX \to \bmL(\bmU)$ of the form $\bmK(x,x')= I_{\bmU} k(x,x')$ where $k$ is the positive definite kernel defined in Section \ref{sec:nonparam} and $I_{\bmU}$ is the identity operator on the Hilbert space $\bmU$, then the solution to the problem:
   $$ \min_{h \in \bmH_{\bmK}} \sum_{i=1}^N \| h(x_i) - \phi(y_i) \|^2_{\bmU} + \lambda \| h \|_{\bmH_{\bmK}^2},$$

for $\lambda > 0$, writes as 
%\begin{equation*}
$\hat{h}(x) = \sum_{i=1}^N \alpha_i(x) \phi(y_i)$
%\end{equation*}
with $\alpha(x)$ verifying Eq. \eqref{eq:krr_weights}. Then, $\hat{f}(x)$ can be expressed as
\begin{align*}
 \arg \min_{z \in \bmZ} \left\{\langle \psi(z),\sum_{i=1}^N \alpha_i(x) \phi(y_i) \rangle  =\sum_{i=1}^N \alpha_i(x) \Delta(z, y_i)  \right\}%\\
  %  \hat{f}(x) &= \arg \min_{z \in \bmZ} \langle \psi(z),\sum_{i=1}^N \alpha_i(x) \phi(y_i) \rangle\\
  % &= \arg \min_{z \in \bmZ} \sum_{i=1}^N \alpha_i(x) \Delta(z, z_i) 
\end{align*}

We show in the following proposition that $\Delta_{FGW}$ admits an ILE. This allows us to obtain theoretical guarantees from \citet{ciliberto2020general} for our estimator.
\begin{proposition} $\Delta_\FGW$ admits an ILE.
\end{proposition}
\begin{proof}
 $\bmY$ is a finite space by definition. $\bmZ_n$ is a compact space as $[0,1]^{n \times n}$ and $\text{Conv}(\bmF)^n$ are compact ($\bmF$ is finite). Moreover, $\forall\, y \in \bmY,  z \rightarrow \Delta_\FGW(z, y)$ is a continuous map (See Lemma \ref{lem:continuity}). Therefore, according to Theorem 7 from \citet{ciliberto2020general} $\Delta_\FGW: \bmZ_n \times \bmY \rightarrow \R$ admits an ILE.
\end{proof}

%Showing that FGW is ILE is a key point to obtain statistical guarantees. Indeed, in this case, the proposed estimator belongs to the family of surrogate structured prediction estimator, as it writes: $\hat f(x) = \argmin_z \langle \psi(z),\, \hat h(x)\rangle_{\bmH}$ with $\hat h(x) = \sum_i \alpha_i(x) \phi(z_i)$ the kernel ridge estimator trained with the points $(x_i, \phi(z_i))_i$. Statistical properties of such estimator have been studied in \cite{ciliberto2020general}.

% \fl{IMPORTANT/ FGW is defined over a metric measure space $\bmZ$ which is compact if the metric space is compact. In principle, we therefore need to impose constraint on symmetric matrices $C$ (for instance bounded values and also on  $F$. Continuity of FGW is ensured.}
% When $\bmZ$ is a space with finite cardinality, one can always writes a map $\Delta: \bmZ \times \bmZ \rightarrow \R$ as a scalar product. Hence, when considering $\bmZ$ as the space of discrete graphs with a limited maximum number of nodes, with binary adjacency matrices as similarity matrices and discrete labels, the FGW distance is ILE \fl{ref to be added: nowak}. In the case of molecule prediction, this finite cardinality assumption is true. Nodes are labeled by atoms which are element of a finite alphabet and edges correspond to chemical bonds between atoms ($C \in \{0,1\}^{n \times n}$).\\

% Nevertheless, we expect this result to extend to continuous label spaces, for example showing that $z \rightarrow GW(z, z')$ is smooth for any $z' \in \bmZ$, and applying the results of \citet{luise:hal-01958887}. We leave it as future work.

\subsection{Excess-risk bounds}
% In this subsection, we show the consistency of the proposed nonparametric estimator with the FGW loss, and provide excess-risk bounds. 
Since $\Delta_{FGW}$ is ILE, the proposed estimator enjoys consistency (See Theorem \ref{th:consistency} in Appendix). Moreover, under an additional technical assumption (Assumption \ref{as:2} in Appendix), it verifies the following excess-risk-bound.

\begin{restatable}[Excess-risk bounds]{theorem}{th2}\label{th:learning_bounds} Let $k$ be a bounded continuous reproducing kernel such that $\kappa^2 := \sup_{x \in \bmX} k(x,x) < +\infty$. Let $\rho$ be a distribution on $\bmX \times \bmY$. Let $\delta \in (0,1]$ and $N_0$ sufficiently large such that $N_0^{-1/2} \geq \frac{9\kappa^2}{N_0}\log\frac{N_0}{\delta}$. Under Assumption \ref{as:2}, for any $N \geq N_0$, if $f_W$ is the proposed estimator built from $N$ independent couples $(x_i, y_i)_{i=1}^N$ drawn from $\rho$. Then, with probability $1-\delta$
\begin{align}
    \mathcal\risk_{\Delta}^n(f_W) - \risk_{\Delta}^n(f^*) \leq c \, \log(4/\delta) \, N^{-1/4},
\end{align}
with $c$ a constant independent of $N$ and $\delta$.
\end{restatable}

Note that $N^{-1/4}$ is the typical rate for structured prediction problems without further assumptions on the problem \citep{ciliberto2016consistent, ciliberto2020general}. Theorem \ref{th:learning_bounds} relies on the attainability assumption \ref{as:2}. This can be interpreted as the fact that the proposed GW barycentric model defines an hypothesis space which is able to deal with graph prediction problems that are smooth with respect to the FGW metric. This corroborates with the intuition that for such problems FGW interpolation will obtain good prediction results. We illustrate this theoretical insight on a synthetic dataset in the experimental section. Furthermore, both theorems are valid for any $\bmZ_n, n \in \mathbb{N}^*$, that is, they provide guarantees for all regression problems defined in Eq. \eqref{eq:exp_structured_risk} for all $n \in \mathbb{N}^*$.

% The proposed model
% can be seen as a strong a priori on the form of $F(x,z;W)$.
% Strong a priori is statistically very beneficial when it's indeed
% a correct assumption on the problem at hand.

\section{Neural network-based conditional Gromov-Wasserstein barycenter}\label{sec:nns}

%\subsection{Neural networks model}\label{subsec:nns}

In this section, we discuss how to train a neural network model estimator as
defined in Equation \eqref{eq:fgw_estimator} where 
the template graphs $\bar z_j$ are learned simultaneously with the weight function $\alpha$. This provides a very generic model that inherits the
flexibility of deep neural networks and their ability to learn input data
representation. 

\rf{\paragraph{Parameters of the model.}
First we recap the different parameters that we want to optimize. First, the weights
$\alpha(x,W)$ of the barycenter are modeled by a deep neural network with
parameters $W$. Next the templates $M$ graphs $\bar z_j$ are also estimated
allowing the model to better adapt to the prediction task. It is
important to note that $M$ is also a parameter of the model that will tune the
complexity of the model and will need to be validated in practice. Note that
this parametric formulation is better suited to large scale datasets since the
complexity of the predictor will be fixed by $M$ instead of increasing with the
number of training data $N$ as in non-parametric models.

\paragraph{Stochastic optimization of the model.}
We optimize the parameters of the model using a classical  ADAM
\cite{kingma2014adam} stochastic
optimization procedure where the gradients are taken over samples or minibatches of
the full empirical distribution.\\ % In our experiments we used the ADAM
%\cite{kingma2014adam} optimizer that provides an adaptive step.\\
We now discuss the computation of the stochastic gradient on a training sample
$(x_i,y_i)$. First note that the gradient of $\FGW(f_\theta(x_i),y_i)$
\emph{w.r.t.} $\theta$ is actually the gradient of a bi-level optimization
problem since $f_\theta$ is the solution of a FGW barycenter. The barycenter
solutions expressed in Equations \eqref{eq:C} and \eqref{eq:A} actually depends
on the optimal OT plans $(\bar \pi_j)_j$ of the barycenter that depends themselves on $\theta$.
But in practice the OT plans $(\bar \pi_j)_j$ are solutions of a non-convex and
non-smooth quadratic program and are with high probability on a border of the
polytope \cite{maron2018probably}. This means that we can assume that a small change in $\theta$ will not
change their value and a reasonable differential of $(\bar
\pi_j)_j$ \emph{w.r.t.} $\theta$ is the null vector. This
actually corresponds in Pytorch \citep{paszke2019pytorch} notation to "detach" the OT plan with respect to the input
which is done by default in POT toolbox {\citep{flamary2021pot}}. The gradient of the outer $\FGW$ loss can be easily computed as the gradient of the loss  with the fixed
optimal plan $\pi_i$ using the theorem from \cite{bonnans1998optimization}.
Computing a sub-gradient of the
loss $\FGW(f_\theta(x_i),y_i)$ can then be done with the following steps:
\vspace{-2mm}
\begin{enumerate}
  \setlength{\itemsep}{2pt}
  \setlength{\parskip}{0pt}
  \setlength{\parsep}{0pt}
    \item $(\bar \pi_j)_j \leftarrow$ Compute the barycenter $f_\theta(x_i)$.
    \item $ \pi_i \leftarrow$ Compute the loss $\FGW(f_\theta(x_i),y_i))$.
    \item $ \nabla_\theta \leftarrow$ Compute the gradient of $\FGW(f_\theta(x_i),y_i))$  with fixed OT plans $(\bar \pi_j)_j$ and $\pi_i$.
\end{enumerate}
\vspace{-2mm}
Note that for the matrices $\bar C_j$ in the templates, the stochastic update is actually a projected gradient step onto the set of matrices with components belonging to $[0,1]$.

}

\section{Numerical experiments}\label{sec:num_exp}

\begin{figure*}[ht!]
\vskip 0.1in
%\begin{center}
    \centering
    %\vspace{-5mm}
\centerline{\includegraphics[width=.95\textwidth]{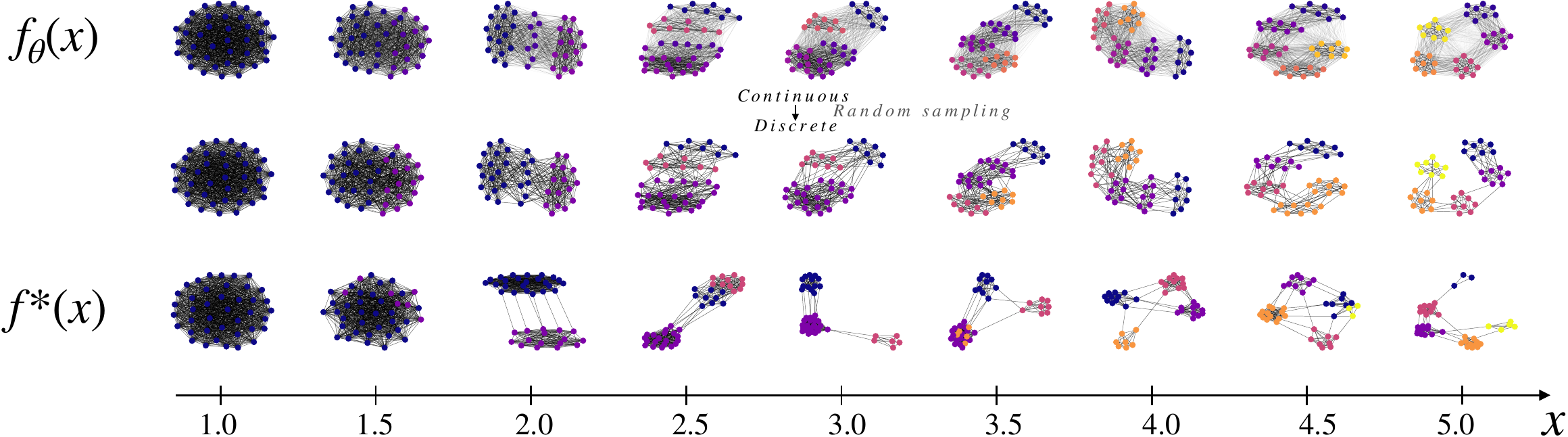}}
\caption{\rf{Graph prediction on the synthetic dataset as a function of the 1D
input $x$.
(top) estimated continuous prediction $f_\theta(x)$, (middle) discrete realizations following
the continuous prediction, (bottom)  true graph prediction function $f^\star (x)$. }}
\label{fig:learned_map}
%\end{center}
%\vskip -0.2in
\vspace{-1mm}
\end{figure*}

In this section, we evaluate the proposed method on a synthetic problem and the metabolite identification problem. %Results show the interest of the proposed graph prediction method.
\cam{A Python implementation of the method is available on github\footnote{\url{https://github.com/lmotte/graph-prediction-with-fused-gromov-wasserstein}}.}

\subsection{Synthetic graph prediction problem}

\paragraph{Problem and dataset. } We consider the following graph prediction
problem. Given an input $x$ drawn uniformly in $[1, 6]$, $y$ is drawn using a
Stochastic block model with $\lfloor x \rfloor$ blocks, such that the biggest
block smoothly splits into two blocks when $x$ is between two integers (see
Figure \ref{fig:learned_map}, bottom line). Each node has a label, which is an
integer indicating the block the node is belonging to. More precisely, we take
randomly from $40$ to $45$ nodes for each graph (uniformly in $\llbracket 40,
45 \rrbracket$. There is a probability $0.9$ of connection between nodes
belonging to the same block, and a probability $0.01$ of connection between
nodes belonging to different blocks. The probability of connection between nodes
belonging to the splitting blocks is $p(x) = 0.889(x - \lfloor x \rfloor) +
0.01$. When a node belongs to the new appearing block its label is the new
block's label with probability $(x - \lfloor x \rfloor)$, and the splitting
block's label otherwise. We generate a training set of $N=50$ couples $(x_i,
y_i)_{i=1}^N$.  Notice that the considered
learning problem is highly difficult as one want to predict a graph from a
continuous value in $[1,6]$.

\paragraph{Experimental setting. } We test the parametric version of the
proposed method with learning of the templates. We use $M=10$ templates, with
$5$ nodes, and initialize them drawing $\bar C_i \in \R^{5 \times 5}, \bar F_i
\in \R^{5 \times 1}$ uniformly in $[0,1]^{5 \times 5}$ and $[0,1]^{5 \times 1}$.
The weights $\alpha(x;W) \in \R^M$ are implemented using a three-layer (\rf{100 neurons in each hidden layer}) fully
connected neural network with ReLU activation functions, and a final softmax
layer. We use $\beta = 1/2$ as FGW's balancing parameter and a prediction size of $n=40$
during training. \rf{During training,  we optimize the parameters $\theta$ of
the model using the continuous relaxed graph prediction model. Interestingly this
prediction provides us with continuous versions of the adjacency matrices so
we can generate discrete graphs by randomly sampling each edge with a Bernouilli
distribution of parameter given by $C(x,\theta)$.}
%We compute predictions in the discrete graph space from the
%predictions in the continuous relaxed graph space by randomly sampling from the  

\paragraph{Supervised learning result. }  \rf{The estimated graph prediction
model on the synthetic dataset is illustrated in  Figure \ref{fig:learned_map}.
We can see that the learned map is 
indeed recovering the evolution of the graphs as a function of $x$. This shows,
as suggested by the theoretical results in Section \ref{sec:nonparam}, that the
FGW metric 
is a a good data fitting term and that FGW barycenters are a good way to
interpolate continuously between discrete objects. This is particularly true on
this problem where a
small change w.r.t $x$ induces small change in the output of $f^*(x)$ according to the FGW
metric. }

\rf{
\paragraph{Interpretability and flexibility of the proposed model. } 
We now illustrate how interpretable is the estimated model.  First we recall that the prediction is actually a Fréchet mean
w.r.t the FGW distance, according to the weights $\alpha_j(x)$ and the templates
$(\bar z_j)_{j=1}^m$. In practice it means that we can plot the template graphs
$(\bar z_j)_{j=1}^m$  to check that the learned
templates are indeed similar (with less nodes) to training
data. But on this synthetic dataset we can also plot the trajectory of the barycenter
weights $\alpha_j$ on the simplex as a function of $x$ which we did in Figure
\ref{fig:simplex}. We can see in the figure that in practice the weights
$\alpha_j(x)$ are sparse concentrated on the templates on the left of the Figure
starting with a graph with one connected cluster and ending with a graph with 5
clusters following the true model $f^\star$.\\
We now illustrate one very interesting property of our model: the ability to
predict graphs with a varying number of nodes $n$ for a given input $x$.
An example of the predicted graphs for $x=5$ is provided in Figure
\ref{fig:nout}. It is interesting to note that even with small templates of
$5$ nodes, the proposed barycentric graph prediction model is able to predict
big graphs while preserving their global structure. 
This is particularly true
for Stochastic Block Models graphs that can by construction be factorized with
a small number of clusters.   Note that the number of nodes
in the templates $(\bar z_j)_{j=1}^m$  can be seen as a regularization parameter. The model is also very flexible in the sens that the
FGW barycenter modeling allows for templates with different number of nodes
allowing for a coarse to fine modeling of the data. 
% In Figure \ref{fig:nout}, we plot the predictions of our model with various $n_{out}$ from $10$ to $130$, for $x=5.0$. It is interesting to note that with small templates the proposed barycentric graph prediction model is able to predict big graphs. Small templates can contains coarse information for big graphs. With the graph distribution considered here (Stochastic block model) one does not need big templates, that would contain finer information. Using small templates in the the proposed model can be statistically beneficial by acting as a regularization.
% model and
% show how the evolution of graphs along $x$ has been 
% In Figure \ref{fig:simplex}, we plot the learned weights $\alpha(x) \in [0,1]^M, \forall x \in [0, 5.5]$ that belongs to the simplex with $M$ points. We plot the corresponding learned templates. We see that, only $7$ templates are activated by $\alpha$.
\begin{figure}[t!]
%\vskip 0.2in
\begin{center}
\centerline{\includegraphics[width=.8\columnwidth]{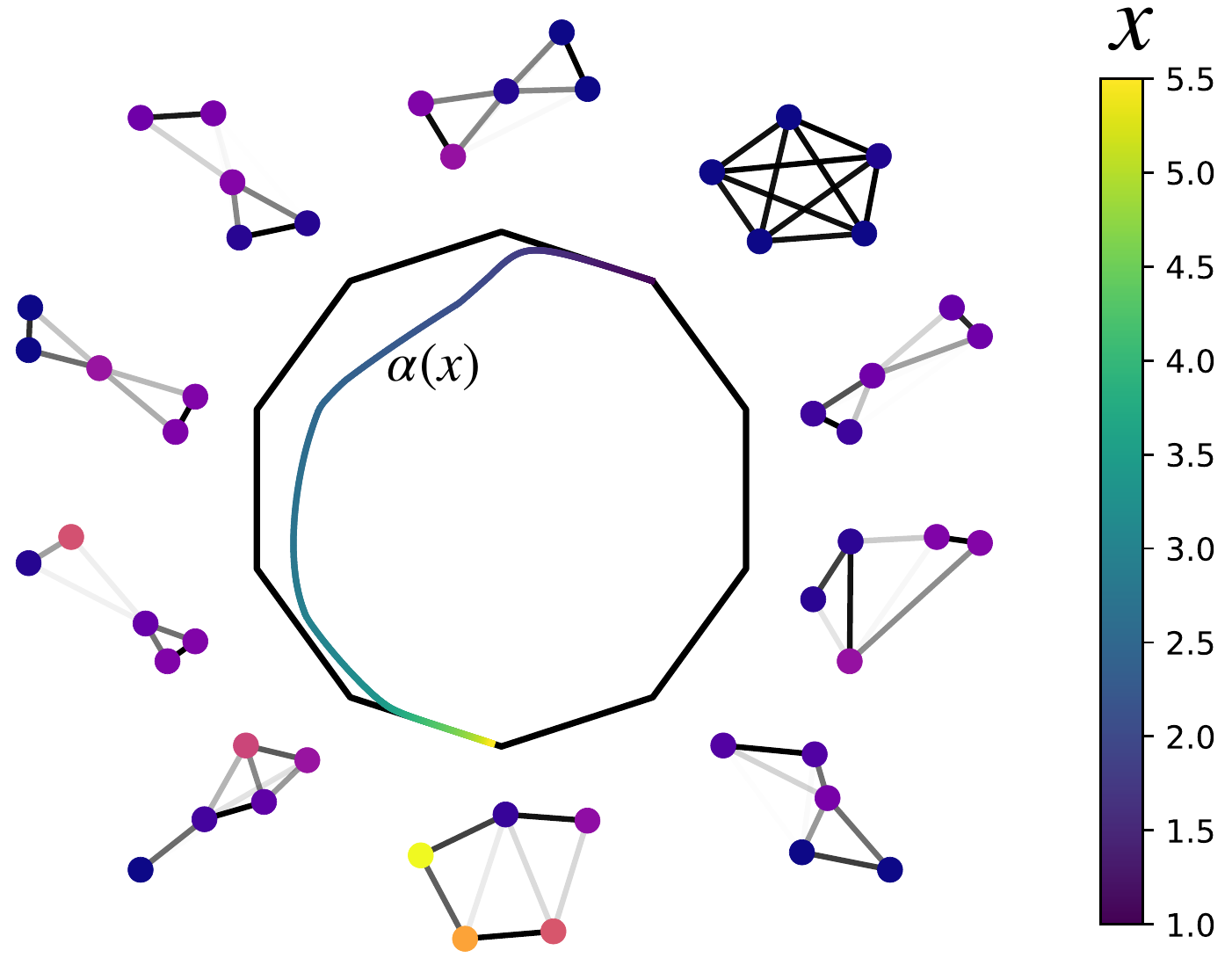}}
\caption{\rf{Learned templates $(\bar z_j)_{j=1}^m$ on the synthetic dataset and trajectory of the
weights $\alpha(x)$ on the simplex as a function of $x$.}}
\label{fig:simplex}
\end{center}
\vskip -0.4in
\end{figure}}
% \lbm{LBM: to put somewhere
% \begin{remark}[Model interpretability.] The proposed method has interesting interpretability properties. For a given input $x$, the predicted graph is the Fréchet mean w.r.t the FGW distance, according to the weights $\alpha_i(x)$ and the templates $(\bar z_i)_{i=1}^m$.
% \end{remark}}
\subsection{Metabolite identification problem}

\paragraph{Problem and dataset. } An important problem in metabolomics is to identify the small molecules, called metabolites, that are present in a biological sample. Mass spectrometry is a widespread method to extract distinctive features from a biological sample in the form of a tandem mass (MS/MS) spectrum. The goal of this problem is to predict the molecular structure of a metabolite given its tandem mass spectrum. Labeled data are expensive to obtain, and despite the problem complexity not many labeled data are available in datasets. Here we consider a set of $4138$ labeled data, that have been extracted and processed in \citet{duhrkop2015searching}, from the GNPS public spectral library \citep{wang2016sharing}. \cam{Datasets and code for reproducing the metabolite identification experiments are available on github\footnote{\footnotesize{\url{lmotte/metabolite-identification-with-fused-gromov-wasserstein}}}.}

\begin{figure}[t!]
%\vskip 0.2in
\begin{center}
\centerline{\includegraphics[width=\columnwidth]{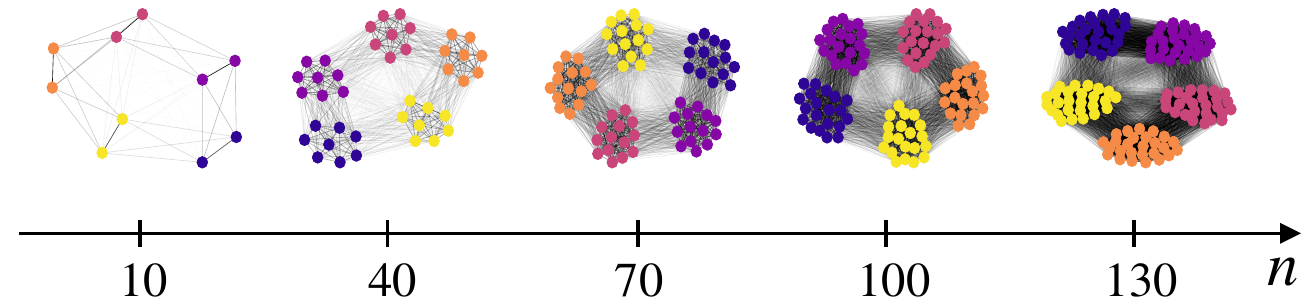}}
\caption{\rf{Predicted graphs with the estimated model $f_\theta(x)$ with a varying number of nodes $n$ for $x=5$.}}
\label{fig:nout}
\end{center}
\vskip -0.2in
\end{figure}

\paragraph{Experimental setting. } We test the nonparametric version of the proposed method, using a probability product kernel on the mass spectra, as it has been shown to be a good choice on this problem \citep{brouard2016fast}. We use $\beta = 0.5$ as FGW balancing parameter. We split the dataset into a training set of size $N=3000$ and a test set of size $N_{te}= 1138$. 
% The molecular structures of the metabolites are represented by fingerprints, that are binary vectors of length d = 7593. Each value of the fingerprint indicates the presence or absence of a certain molecular property. Labeled data are expensive to obtain, and despite the problem complexity only n = 6974 labeled data are available. State-of-the-art results for this problem have been obtained with the IOKR method by Brouard et al. (2016a). The median size of the candidate sets is 292, and the biggest candidate set is of size 36918. Hence, the metabolite identification dataset is characterized by high-dimensional complex outputs, a small training set, and a very large number of candidates.
% \fl{you are talking about accuracies: meaning that you have discretized your prediction before applying a zero-one accuracy metrics ?} \lbm{LBM: no need as I compute the barycenter over the candidates that are already discrete graphs.}
%\paragraph{Graph metrics comparison: experimental setting.} 
 On
this problem, structured prediction approaches that have been proposed fall back
on the availability of a known candidate set of output graphs for each input
spectrum \citep{brouard2016fast}. \rf{This means that in practice for prediction
on new data, we will not
solve the FGW barycenter in \eqref{eq:fgw_estimator} but search among the possible
candidates in $\bmY$ the one minimizing the barycenter loss.}\\
In a first experiment, we evaluate the performance of FGW as a graph metric.
To this end we compare the performance of various graph metrics $D: \bmY \times \bmY \rightarrow \R^+$ used in the model: $\argmin_{y \in \bmY} \sum_{j=1}^N \alpha_j(x; W) D(y, y_j)$. We consider the metric induced by the standard Weisfeiler–Lehman (WL) graph kernel that consists in embedding graphs as a bag of neighbourhood configurations \citep{shervashidze2011weisfeiler}. The FGW one-hot distance corresponds to the FGW distance and using a one-hot encoding of the atoms. The FGW fine distance corresponds to the one-hot distance concatenated with additional atom features: number of attached hydrogens, number of heavy neighbours, formal charge, is in a ring, is in an aromatic ring. Additional features are normalize by their maximum values in the molecule at hand. The FGW diffuse distance corresponds to the FGW distance and using a one-hot encoding of the atoms which has been diffused, namely: $F_{\text{diff}} = e^{- \tau \text{Lap}(C)}F$, where $\tau>0$, $\text{Lap}(C)$ denotes the normalized Laplacian of $C$ as proposed in \citet{barbe2020graph}. Fingerprints are molecule representations, well engineered by experts, that are binary vectors. Each value of the fingerprint indicates the presence or absence of a certain molecular property (generally a molecular substructure). Several machine learning approaches  using fingerprints as output representations have obtained very good performances for metabolite identification  \citep{duhrkop2015searching, brouard2016fast,nguyen2018} or other tasks, such as metabolite structural annotation \citep{hoffmann2021high}. In the last two Casmi challenges \citep{schymanski2017critical}, such approaches have obtained the best performances for the best automatic structural identification category.
%that is why we use them as our competitors
Here we consider the metrics induced by linear and Gaussian kernels between fingerprints of length $d = 2765$. \cam{Notice that, in this case, the structured prediction method corresponds to IOKR-Ridge proposed in \citet{brouard2016input}}.
%We compute the test predictions using the test spectra with less than 300 candidates for faster computation: 286 test points.
For the FGW metrics, we compute them using the 5 greatest weights $\alpha_i(x)$. We evaluate the results in terms of Top-k accuracy: percentage of true output among the k outputs given by the k greatest scores in the model. The two hyperparameters (ridge regularization parameter $\lambda$ and the output metric's parameter) are selected using a validation set (1/5 of the training set) and Top-1 accuracy.
% 
% \begin{table}[t!]
% \caption{Top-k accuracies for various graph kernels on the metabolite identification dataset.}
% \label{tab:metric_comparison}
% \vskip 0.15in
% \begin{center}
% \begin{small}
% \begin{sc}
% \begin{tabular}{lccc}
% \toprule
%  & Top-1 & Top-10 & Top-20\\
% \midrule
% WL kernel & 28.7\% & 57.2\% & 71.2\% \\
% Linear fingerprint & 33.6\% & 76.2\% & 80.1\%\\
% Gaussian fingeprint & 48.7\% & 81.0\% & 86.0\% \\\hline
% FGW one-hot & 24.6\% & 64.2\% & 75.4\% \\
% FGW fine & 31.2\% & 64.9\% & 76.1\%\\
% FGW diffuse & 40.0\% & 72.3\% & 82.5\%\\

% \bottomrule
% \end{tabular}
% \end{sc}
% \end{small}
% \end{center}
% \vskip -0.1in
% \end{table}
%
\cam{\begin{table}[t!]
\begin{center}
\begin{small}
\begin{sc}
\begin{tabular}{lccc}
\toprule
 & Top-1 & Top-10 & Top-20\\
\midrule
WL kernel & 9.8\% & 29.1\% & 37.4\% \\
Linear fingerprint & 28.6\% & 54.5\% & 59.9\%\\
Gaussian fingeprint & 41.0\% & 62.0\% & 67.8\% \\\hline
FGW one-hot & 12.7\% & 37.3\% & 44.2\% \\
FGW fine & 18.1\% & 46.3\% & 53.7\%\\
FGW diffuse & 27.8\% & 52.8\% & 59.6\%\\
\bottomrule
\end{tabular}
\end{sc}
\end{small}
\end{center}
\caption{Top-k accuracies for various graph \cam{metrics} on the metabolite identification dataset.}
\label{tab:metric_comparison}
\end{table}}
\paragraph{Graph metrics comparison.} The results given in Table \ref{tab:metric_comparison} shows that Gaussian fingerprints is the best performing metric on this dataset when a candidate set is available.
% We see that the FGW greatly benefits from the improved fine and diffuse metrics showing the adaptation potential of the FGW metric to the graph space at hand reaching competitive performance against baselines and even beating Fingerprints with linear kernel and WL kernels.
\cam{We see that the FGW greatly benefits from the improved fine and diffuse metrics showing the adaptation potential of the FGW metric to the graph space at hand reaching competitive performance against fingerprints with linear kernel and beating WL kernels. The method proposed in this work is the first generic approach that obtained good Top-k accuracies without using expert-derived molecular graph representations.}

\paragraph{Predicting novel molecules.} Being able to interpolate novel graphs without using predefined \cam{finite} candidate sets is a great advantage of the proposed method. Such computation is in general intractable (e.g. with WL and fingerprint metrics). In this experiment, we evaluate the performance of the estimator when computing the barycenter over $\bmZ_n$, and not over the candidate sets. For a given test input $x$, let us define $d_0(x)$ the FGW (one-hot) distance of the training molecule with the greatest $\alpha_j(x)$ to the true molecule. $d_0(x)$ measures the level of interpolation difficulty: very small $d_0$ means that the true molecule is close to a training molecule and no interpolation is required. We compute, over $1000$ test data, the mean $d_0(x)$ and the mean FGW (one-hot) distance between the predicted barycenter (using the $10$ largest $\alpha_j(x)$) and the true test molecule. In Figure \ref{fig:novel}, we plot the two mean distances, with respect to a filtering threshold $d_{min}$ such that only the test point with $d_0(x) > d_{min}$ are used when computing these means. We can see that the FGW interpolation allows to become closer to the true output than only predicting the output with the greatest weight $\alpha_j(x)$, even more when interpolation is required ($d_0(x)$ big). This validates the choice of FGW as a way to interpolate between real-world graphs.
\begin{figure}[t!]
%\vskip 0.2in
\begin{center}
\centerline{\includegraphics[width=0.6\columnwidth]{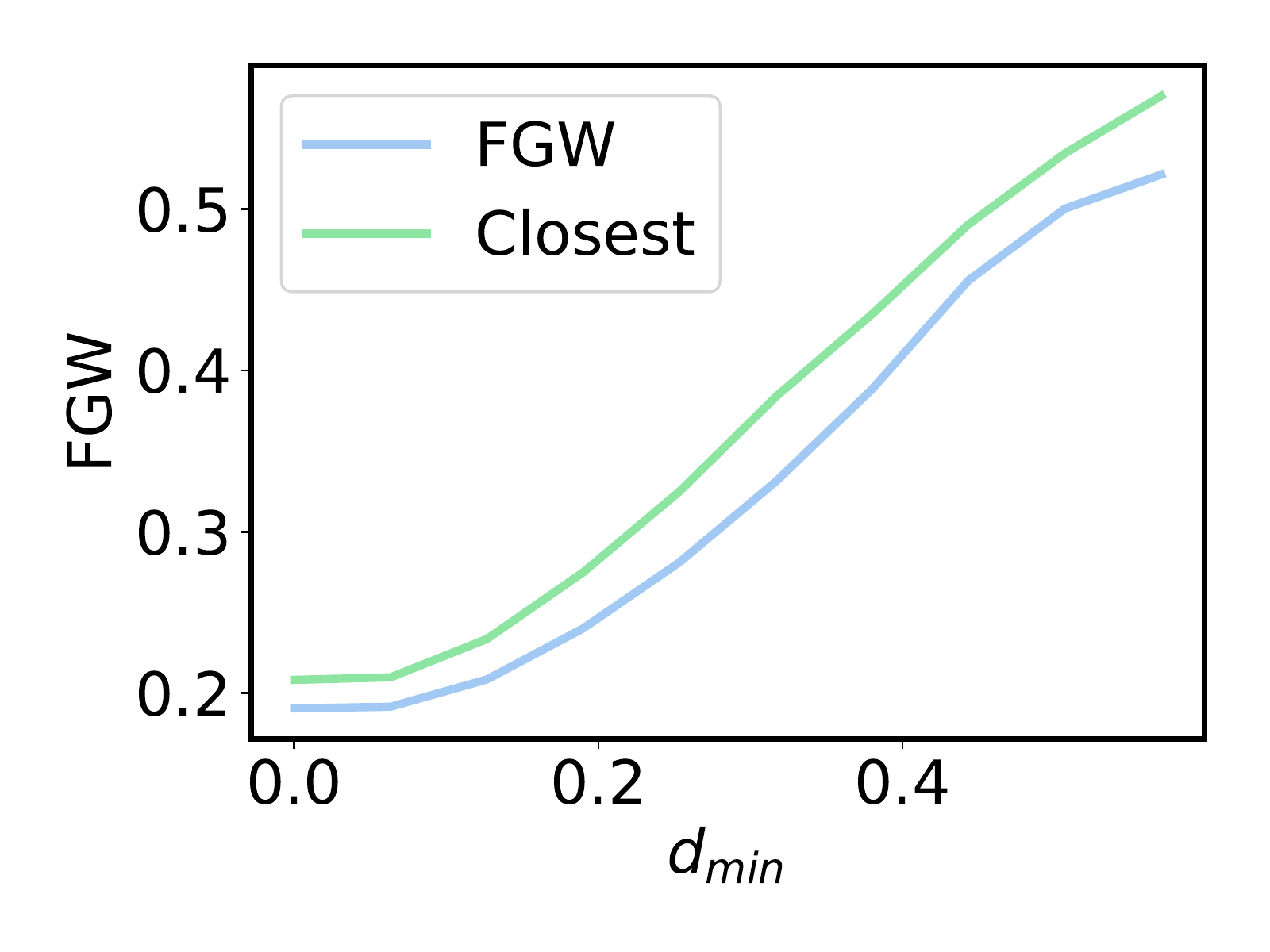}}\vspace{-4mm}
\caption{No candidate set setting. In average, the FGW barycenter (blue) using the 10 molecules with the greatest weights $\alpha_j(x)$ is closer to the true molecule, than the molecule with the greatest weight $\alpha_j(x)$: closest template prediction (green).}
\label{fig:novel}
\end{center}
\vskip -0.3in
\end{figure}
\cam{
\paragraph{Comparison with a flow-based deep graph generation method.} As mentioned previously, to the best of our knowledge, there is no generic method for graph prediction able to deal with any graph space at hand. The only existing methods, that do not require expert-derived graph representations available for a specific graph space, are unsupervised deep graph generation methods \citep{li2018multi, liao2019efficient, zang2020moflow, mercado2021graph}. We propose to compare our approach by designing a new generic graph prediction method. We use the deep generative graph representations from MoFlow \citep{zang2020moflow} learned from 249.455 molecules and which obtained state-of-the-art results in (unsupervised) molecular graph generation. The latent representations are learned via kernel ridge regression, then we predict the candidate with the closest latent representation to the estimated one. Note that because the pre-trained model's architecture can not handle all atoms present in the metabolite dataset, we removed from the dataset the molecules with not handled atoms. Moreover, we compute the test predictions using the test spectra with less than 300 candidates for faster computation: 286 test points. The results are given in Table \ref{tab:moflow_metric_comparison}. We observe that FGW diffuse exhibits far better performance than the MoFlow approach.
\begin{table}[ht!]
\begin{center}
\begin{small}
\begin{sc}
\begin{tabular}{lccc}
\toprule
 & Top-1 & Top-10 & Top-20\\
\midrule
Gaussian fingerprint & 46.2\% & 77.8\% & 84.9\%\\\midrule
FGW diffuse & 40.3\% & 69.7\% & 78.3\%\\
MoFlow representat. & 20.0\% & 58.2\% & 68.4\%\\
\bottomrule
\end{tabular}\vspace{-2mm}
\end{sc}
\end{small}
\end{center}
\caption{Top-k accuracies obtained using deep molecular graph representations in comparison to the proposed FGW metric, and expert-derived fingerprint representations.}\vspace{-6mm}
\label{tab:moflow_metric_comparison}
\end{table}\\}

\section{Conclusion}

\rf{We proposed in this work a novel framework for graph prediction using optimal transport barycenters to interpolate continuously in the output space. We discussed both a non-parametric estimator with theoretical guarantees and a parametric one based on neural network models that can be estimated with stochastic gradient methods. The method was illustrated on synthetic and real life data showing the interest of the continuous relaxation especially when targets are not available.\\
Future works include estimation of the target number of nodes $n(x)$ and supervised learning of complementary feature on the templates that can guide the FGW barycenters. }

\section*{Acknowledgements}
The first and last authors are funded by  the French National Research Agency (ANR) through ANR-18-CE23-0014 APi (Apprivoiser la Pré-image) and the Télécom Paris Research Chair DSAIDIS. This work was also partially funded through the projects OATMIL ANR-17-CE23-0012, 3IA Côte d’Azur Investments ANR-19-P3IA-0002 of the French National Research Agency (ANR) and was produced within the framework of Energy4Climate Interdisciplinary Center (E4C) of IP Paris and Ecole des Ponts ParisTech. It was supported by 3rd Programme d’Investissements d’Avenir ANR-18-EUR-0006-02. This action benefited from the support of the Chair "Challenging Technology for Responsible Energy" led by l’X – Ecole polytechnique and the Fondation de l’Ecole polytechnique, sponsored by TOTAL. This research was partially funded by Academy of Finland grant 334790 (MAGITICS).

%\newpage
\bibliography{references}
\bibliographystyle{icml2022}

%%%%%%%%%%%%%%%%%%%%%%%%%%%%%%%%%%%%%%%%%%%%%%%%%%%%%%%%%%%%%%%%%%%%%%%%%%%%%%%
% APPENDIX
%%%%%%%%%%%%%%%%%%%%%%%%%%%%%%%%%%%%%%%%%%%%%%%%%%%%%%%%%%%%%%%%%%%%%%%%%%%%%%%
\newpage
\appendix
\onecolumn

\section{Theory}

\subsection{Proof of FGW continuity}

We prove the continuity of $\FGW(., y) : \bmZ_n \rightarrow \R$ for any $y \in \bmY$. Such result is crucial to prove the ILE property of $\FGW: \bmZ_n \times \bmY \rightarrow \R$.

\begin{restatable}[FGW continuity]{lemma}{lem1}\label{lem:continuity} Let $y = (C_2, F_2)$ with $C_2 \in \R^{n_2 \times n_2}, F_2\in \R^{n_2 \times d}$, $n_2, d \in \mathbb{N}^*$. The map $\FGW(., y) : \bmZ_n \rightarrow \R$ is continuous.
\end{restatable}

\begin{proof} Recall that for any $z=(C, F) \in \bmZ_n$:
\begin{align}
    \FGW^2_2(z, y) = \min\limits_{\pi \in \plans_{n1,n2}} \sum_{i,k,j,l} \big[(1-\beta)\|F(i) - F_2(j)\|_{\R^d}^2 + \beta (C(i,k) - C_2(j,l))^2\big] \pi_{i,j}\pi_{k,l}.
\end{align}
Using the inequality $|\min_\pi f(\pi) - \min_\pi g(\pi)| \leq \sup_\pi |f(\pi) - g(\pi)|$ for any $f, g: \plans_{n1,n2} \rightarrow \R$, we have for any $dz = (dC, dF) \in \bmZ_n$

\begin{align}
|\FGW^2_2(z + dz, y) - \FGW^2_2(z, y)| 
\leq& \sup_{\pi \in \plans_{n1,n2}} |\sum_{i,k,j,l} \big[(1-\beta)\left(\langle dF(i)|F_2(j)\rangle_{\R^d} + o(\|dF(i)\|\rangle_{\R^d})\right)\\ 
& + \beta \left(dC(i, k) C_2(j,l) + o(dC(i,k))\right)\big] \pi_{i,j}\pi_{k,l}|\nonumber\\
&\leq nn_2 \big[(1-\beta)\left(\|dF\|_{\R^{n \times d}} \|F_2\|_{\R^{n \times d}} + o(\|dF\|_{\R^{n \times d}})\right)\\
&+ \beta \left(\|dC\|_{\R^{n \times n}}\|dC_2\|_{\R^{n_2 \times n_2}} + o(\|dC\|_{\R^{n \times n}})\right)\big]\nonumber\\
&= \mathcal{O}(\|dz\|_{\R^{n \times n} \times \R^{n \times d}}) \xrightarrow[dz \rightarrow 0]{} 0
\end{align}
where from (13) to (14) we have used the Cauchy–Schwarz inequality, and the fact that $\forall (i,j) \in \llbracket 1, n\rrbracket \times \llbracket 1, n_2\rrbracket, \pi_{ij} \leq 1$.

We conclude that $z \rightarrow \FGW^2_2(z, y)$ is a continuous on $\R^{n \times n} \times \R^{n \times d}$, hence on $\bmZ_n$.

% \begin{split|\FGW^2_2(z + dz, y) - \FGW^2_2(z, y)| &\leq a\\
%   \sup_{\pi \in \plans_{n1,n2}} |\sum_{i,k,j,l} \big[(1-\beta)\langle dF(i)|F_2(j) \rangle \\ &+ o(\|dF(i)\|^2)) + \beta \langle dC(i, k)|C_2 \rangle + o(\|dC(i,k)\|^2))\big] \pi_{i,j}\pi_{k,l}|
%     \end{split}
\end{proof}

\subsection{Universal consistency theorem}

We restate the universal consistency theorem from \citet{ciliberto2020general} that is verified by our estimator because of the proved ILE property.

\begin{restatable}[Universal Consistency]{theorem}{th1}\label{th:consistency} Let $k$ be a bounded universal reproducing kernel. For any $N \in \mathbb{N}$ and any distribution $\rho$ on $\bmX \times \bmY$ let $f_W$ be the proposed estimator built from $N$ independent couples $(x_i, y_i)_{i=1}^N$ drawn from $\rho$. Then, if $\lambda = N^{-1/2}$,
\begin{equation}
    \lim\limits_{N \rightarrow + \infty} \risk_{\Delta}^n(f_W) = \risk_{\Delta}^n(f^*) \quad \text{ with probability } \quad 1.
\end{equation}
\end{restatable}

\subsection{Attainability assumption}

The following assumption is required to obtain finite sample bounds. It is a standard assumption in learning theory \citep{caponnetto2007optimal}. It corresponds to assume that the solution $h^*$ of the surrogate problem indeed belongs to the considered hypothesis space, namely the reproducing kernel Hilbert space induced by the chosen operator-valued kernel $\bmK(x,x') = k(x,x') I_\bmU$.

\begin{assumption}[attainable case]\label{as:2} We assume that there exists a linear operator $H: \bmH_x \rightarrow \bmU$ with $\|H\|_{\hs} < +\infty$ such that
\begin{equation}
    \E_{Y | x}[\phi(Y)] = H k(x,.)
\end{equation}
with $\bmH_x$ the reproducing kernel Hilbert space associated to the kernel $k(x,x')$.
\end{assumption}

\section{Neural network model and training algorithm}

\paragraph{Choice of the templates.} As always in deep learning, parameter
initialization is an important aspect and we discuss now how to initialize the
templates $\bar z_j$. In practice they can be initialized at random with
matrices $\bar C_j$ drawn uniformly in $[0,1]$ or chosen at random from training
samples as suggested by the non-parametric model. One interesting aspect is that
the number of nodes do not need to be the same for all templates. This means
that one can have both templates with few nodes and templates with a larger
number of nodes allowing for a coarse to-fine modeling of the graphs. %Indeed
%templates with few nodes will correspond to clustering of nodes as shown in the
%simulated data in the next section. 

\paragraph{Pseudocode.} We give the pseudocode for the proposed neural network training algorithm. This algorithm has been implemented in Python using the POT library: Python Optimal Transport \citep{flamary2021pot}, and Pytorch library \citep{paszke2019pytorch}.

\begin{algorithm}\label{algo:train}
\begin{algorithmic}
   \STATE {\textbf{Input:}} $x \rightarrow \alpha(x)$ neural network's parameters $W$. Templates $(\bar z_j)_{j=1}^M$. Dictionary learning (True or False).
   \STATE 1. If Dictionary learning is True: $\theta = (W, (\bar z_j)_{j=1}^M)$. Otherwise: $\theta = W$.
   \STATE 2. $(\bar \pi_j)_{j=1}^M \leftarrow$ Compute the barycenter $f_{\theta}(x_i)$.
   \STATE 3. $ \pi_i \leftarrow$ Compute the losses $\FGW(f_\theta(x_i),y_i))$.
   \STATE 4. $ \nabla_\theta \leftarrow$ Compute the gradient of $\FGW(f_\theta(x_i),y_i))$  with fixed OT plans $(\bar \pi_j)_j$ and $\pi_i$.
  \STATE {\textbf{Return:}} Updated neural network's parameters $W$, updated templates $(\bar z_j)_{j=1}^M$.
\end{algorithmic}
\caption{Neural network-based model training - One stochastic gradient descent step}
\end{algorithm}

\paragraph{Python implementation on github.} The code is available on github at \url{https://github.com/lmotte/graph-prediction-with-fused-gromov-wasserstein}.

\section{Justification of the algorithms}
Reminder on ILE and surrogate problem:\\
Recall that $\hat h$ is solving a least-squares problem, that is estimate $h^*(x) = \E_{z|x}[\phi(z)]$. Moreover, we can write $f^*(x) = \argmin_{\hat z} \E_{z|x}[\Delta(\hat z,z)]$. Now, we can provide intuition in the following derivations about the construction of $\hat f$ exploiting the linearity of expectation.
\begin{align*}
    \hat f(x) &= \argmin_{\hat z} \langle \psi(\hat z),\, \hat h(x)\rangle_{\bmH}\\
    &\approx \argmin_{\hat z} \langle \psi(\hat z),\, h^*(x)\rangle_{\bmH}.
\end{align*}
Moreover, we have: 
\begin{align*}
\langle \psi(\hat z),\, h^*(x)\rangle_{\bmH}  &= \E_{z|x}[\langle \psi(\hat z),\, \phi(z)\rangle_{\bmH}]\\
    &=  \E_{z|x}[\Delta(\hat z ,z)] \\
     \end{align*}
    and thus, taking the "arg min" gives:
\begin{align*}
         \hat f(x) \approx f^*(x).
\end{align*}

\section{Discussion about keeping only the greatest weights $\alpha_i(x)$ in the barycenter computation}

In the metabolite identification experiments we computed the barycenter only using the 5 greatest ones. In the following experiments, we show that, beyond the considerable computational interest, this approximation is also statistically beneficial on this dataset. We compute the test Top-k accuracies by changing the number of kept $\alpha_i$. From Figure \ref{fig:kept_alpha}, it seems that the best number of kept $\alpha_i(x)$ seems to be around $10$.

\begin{figure}[ht!]
\vskip 0.2in
\begin{center}
\centerline{\includegraphics[width=\columnwidth/2]{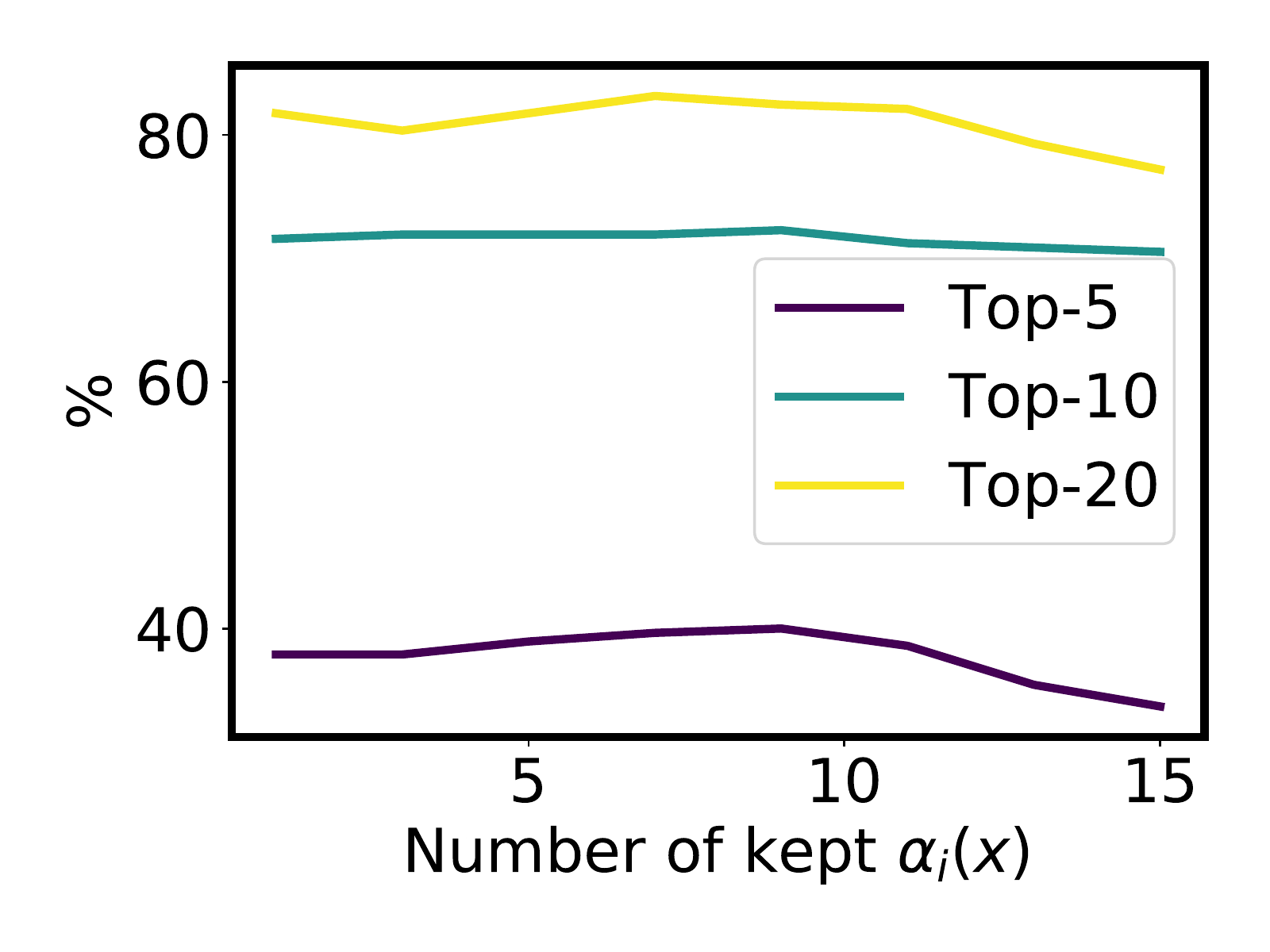}}
\caption{Top-k accuracies of $f_\theta(x)$ using a varying number of kept $\alpha_i(x)$.}
\label{fig:kept_alpha}
\end{center}
\vskip -0.2in
\end{figure}

%%%%%%%%%%%%%%%%%%%%%%%%%%%%%%%%%%%%%%%%%%%%%%%%%%%%%%%%%%%%%%%%%%%%%%%%%%%%%%%

\end{document}